%% file: main_arxiv.tex
\documentclass{article} 

\usepackage{microtype}
\usepackage{graphicx}
\usepackage{subfigure}
\usepackage{booktabs} 

\input{math_commands.tex}

\usepackage{xcolor}
\usepackage{hyperref}
\usepackage{url}

\usepackage[accepted]{icml_arxiv}

\usepackage{caption}
\usepackage{amstext,amsmath,amssymb,amsthm,bm}
\usepackage{nicefrac} 
\usepackage{mathtools}

\newtheorem{theorem}{Theorem}[section]

\icmltitlerunning{Revisiting Prioritized Experience Replay: A Value Perspective}

\begin{document}

\twocolumn[
\icmltitle{Revisiting Prioritized Experience Replay: A Value Perspective}

\begin{icmlauthorlist}
\icmlauthor{Ang A. Li}{pku}
\icmlauthor{Zongqing Lu}{pku}
\icmlauthor{Chenglin Miao}{pku}
\end{icmlauthorlist}

\icmlaffiliation{pku}{Peking University}

\icmlcorrespondingauthor{Zongqing Lu}{zongqing.lu@pku.edu.cn}
\icmlcorrespondingauthor{Chenglin Miao}{chenglin.miao@pku.edu.cn}

\icmlkeywords{Machine Learning, ICML}

\vskip 0.3in

]

\printAffiliationsAndNotice{}
\begin{abstract}
Experience replay enables off-policy reinforcement learning (RL) agents to utilize past experiences to maximize the cumulative reward. Prioritized experience replay that weighs experiences by the magnitude of their temporal-difference error ($|\text{TD}|$) significantly improves the learning efficiency. But how $|\text{TD}|$ is related to the importance of experience is not well understood. We address this problem from an economic perspective, by linking $|\text{TD}|$ to \textit{value of experience}, which is defined as the value added to the cumulative reward by accessing the experience. We theoretically show the value metrics of experience are upper-bounded by $|\text{TD}|$ for Q-learning. Furthermore, we successfully extend our theoretical framework to maximum-entropy RL by deriving the lower and upper bounds of these value metrics for soft Q-learning, which turn out to be the product of $|\text{TD}|$ and ``on-policyness" of the experiences. Our framework links two important quantities in RL: $|\text{TD}|$ and value of experience. We empirically show that the bounds hold in practice, and experience replay using the upper bound as priority improves maximum-entropy RL in Atari games.
\end{abstract}

\section{Introduction}
Learning from important experiences prevails in nature. In rodent hippocampus, memories with higher importance, such as those associated with rewarding locations or large reward-prediction errors, are replayed more frequently \citep{Michon2019, roscow2019behavioural, salvetti2014role}. Psychophysical experiments showed that participants with more frequent replay of high-reward associated memories show better performance in memory tasks \citep{Gruber2016, schapiro2018human}. 
As accumulating new experiences is costly, utilizing valuable past experiences is a key for efficient learning \citep{olafsdottir2018role}.

Differentiating important experiences from unimportant ones also benefits reinforcement learning (RL) algorithms \citep{katharopoulos2018not}. Prioritized experience replay (PER) \citep{Schaul2016} is an experience replay technique built on deep Q-network (DQN) \citep{Mnih2015}, which weighs the importance of samples by the magnitude of their temporal-difference error ($|\text{TD}|$). As a result, experiences with larger $|\text{TD}|$ are sampled more frequently. PER significantly improves the learning efficiency of DQN, and has been adopted \citep{Hessel2018,Horgan2018, Kapturowski2019} and extended \citep{daley2019reconciling, pan2018organizing, schlegel2019importance} by various deep RL algorithms. $|\text{TD}|$ quantifies the unexpectedness of an experience to a learning agent, and biologically corresponds to the signal of reward prediction error in dopamine system \citep{schultz1997neural, glimcher2011understanding}. 
However, how $|\text{TD}|$ is related to the importance of experience in the context of RL is not well understood.

We address this problem from an economic perspective, by linking $|\text{TD}|$ to \textit{value of experience} in RL. Recently in neuroscience field, a normative theory for memory access, based on Dyna framework \citep{sutton1990integrated}, suggests that a rational agent should replay the experiences that lead to most rewarding future decisions \citep{Mattar2018}. Follow-up research shows that optimizing the replay strategy according to the normative theory has advantage over prioritized experience replay with $|\text{TD}|$ \citep{zha2019experience}. Inspired by \cite{Mattar2018}, we define the value of experience as the increase in the expected cumulative reward resulted from updating on the experience. The value of experience quantifies the importance of experience from first principles: assuming that the agent is economically rational and has full information about the value of experience, it will choose the most valuable experience to update, which leads to most rewarding future decisions. As supplements, we derive two more value metrics, which correspond to the evaluation improvement value and policy improvement value due to update on an experience.

In this work, we mathematically show that these value metrics are upper-bounded by $|\text{TD}|$ for Q-learning. Therefore, $|\text{TD}|$ implicitly tracks the value of experience, and accounts for the importance of experience. We further extend our framework to maximum-entropy RL, which augments the reward with an entropy term to encourage exploration \citep{Haarnoja2017}. We derive the lower and upper bounds of these value metrics for soft Q-learning, which are related to $|\text{TD}|$ and ``on-policyness" of the experience. 
Experiments in grid-world maze and CartPole support our theoretical results for both tabular and function approximation RL methods, showing that the derived bounds hold in practice. Moreover, we show that experience replay using the upper bound as priority improves maximum-entropy RL (\textit{i.e.}, soft DQN) in Atari games.   

\section{Motivation}

\subsection{Q-learning and Experience Replay}
\label{qlearning}
We consider a Markov Decision Process (MDP) defined by a tuple $\{\mathcal{S},\mathcal{A},\mathcal{P}, \mathcal{R}, \gamma\}$, where $\mathcal{S}$ is a finite set of states, $\mathcal{A}$ is a finite set of actions, $\mathcal{P}$ is the transition function, $\mathcal{R}$ is the reward function, and $\gamma \in [0, 1]$ is the discount factor. A policy $\pi$ of an agent assigns probability $\pi(a|s)$ to each action $a \in \mathcal{A}$ given state $s \in \mathcal{S}$. The goal is to learn an optimal policy that maximizes the expected discounted return starting from time step $t$, $G_t = \sum_{i = 0}^{\infty}\gamma^i r_{t+i}$, where $r_t$ is the reward the agent receives at time step $t$. Value function $v_\pi(s)$ is defined as the expected return starting from state $s$ following policy $\pi$, and Q-function $q_\pi(s, a)$ is the expected return on performing action $a$ in state $s$ and subsequently following policy $\pi$. 

According to Q-learning \citep{Watkins1992}, the optimal policy can be learned through policy iteration: performing policy evaluation and policy improvement interactively and iteratively. For each policy evaluation, we update $Q(s,a)$, an estimate of $q_\pi(s, a)$, by 
\begin{equation}\nonumber
    Q_\text{new}(s, a) = Q_\text{old}(s, a) + \alpha \text{TD}(s, a, r, s'),
\end{equation}
where TD error $\text{TD}(s, a, r, s') = r + \gamma \max_{a'} Q_\text{old}(s', a') -  Q_\text{old}(s, a)$ and $\alpha$ is the step-size parameter. $Q_\text{new}$ and $Q_\text{old}$ denote the estimated Q-function before and after the update respectively. And for each policy improvement, we update the policy from $\pi_{\text{old}}$ to $\pi_{\text{new}}$ according to the newly estimated Q-function,
\begin{equation}\nonumber
    \pi_{\text{new}}  = \mathop{\argmax}_{a}   Q_\text{new}(s,a).
\end{equation}
Standard Q-learning only uses each experience once before disregarded, which is sample inefficient and can be improved by \textit{experience replay} technique \citep{lin1992self}. We denote the experience that the agent collected at time $k$ by a tuple $e_k = \{s_k, a_k, r_k, s_{k}' \}$. According to experience replay, the experience $e_k$ is stored into the replay buffer and can be accessed multiple times during learning. 

\subsection{Value Metrics of Experience}

To quantify the importance of experience, we derive three value metrics of experience. The utility of update on experience $e_k$ is defined as the value added to the cumulative discounted rewards starting from state $s_k$, after updating on $e_k$. Intuitively, choosing the most valuable experience for update will yield the highest utility to the agent. We denote such utility as the expected value of backup $\text{EVB}(e_k)$ \citep{Mattar2018},
\begin{align}
\text{EVB}(e_k) &= v_{\pi_{\text{new}}}(s_k) - v_{\pi_{\text{old}}}(s_k) \nonumber \\
 &=  \sum_{a}{\pi_{\text{new}}(a|s_k)q_{\pi_{\text{new}}}(s_k,a)} \nonumber \\
 & \qquad\qquad-\sum_{a} {\pi_{\text{old}}(a|s_k)q_{\pi_{\text{old}}}(s_k,a)},  \label{eq:evb}
\end{align}
where $\pi_{\text{old}}$, $v_{\pi_{\text{old}}}$ and $q_{\pi_{\text{old}}}$ are respectively the policy, value function and Q-function before the update, and $\pi_{\text{new}}$, $v_{\pi_{\text{new}}}$, and $q_{\pi_{\text{new}}}$ are those after. 
As the update on experience $e_k$ consists of policy evaluation and policy improvement, the value of experience can further be separated to evaluation improvement value $\text{EIV}(e_k)$ and policy improvement value $\text{PIV}(e_k)$ by rewriting (\ref{eq:evb}):
\begin{multline}\label{eq:evb2}
\text{EVB}(e_k) = \underbrace{\sum_{a}{[\pi_{\text{new}}(a|s_k) - \pi_{\text{old}}(a|s_k)]q_{\pi_{\text{new}}}(s_k,a)}}_{\text{PIV}(e_k)}+ \\
 \underbrace{\sum_{a}{\pi_{\text{old}}(a|s_k)[q_{\pi_{\text{new}}}(s_k,a)-q_{\pi_{\text{old}}}(s_k,a)]}}_{\text{EIV}(e_k)},
\end{multline}
where $\text{PIV}(e_k)$ measures the value improvements due to the change of the policy, and $\text{EIV}(e_k)$ captures those due to the change of evaluation. Thus, we have three metrics for the value of experience: $\text{EVB}$,  $\text{PIV}$ and  $\text{EIV}$.

\begin{figure*}[t!]
  \centering
  \includegraphics[width=.95\textwidth, clip=true, trim = 0mm 8mm 0mm 8mm]{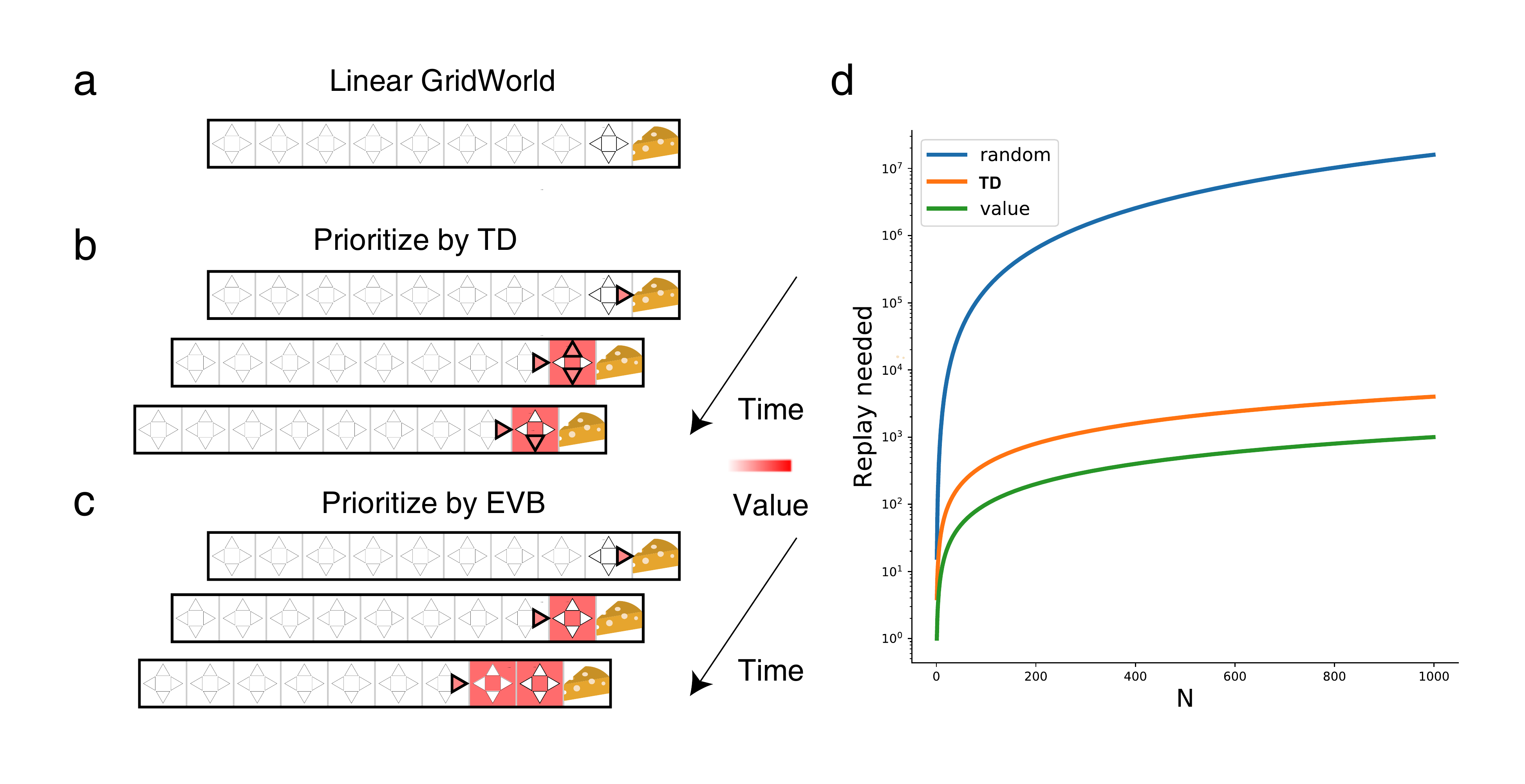}
  \caption{\textbf{a.} Illustration of the ``Linear Grid-World" example: there are $N$ grids and 4 actions (north, south, east, west). Reward for entering the goal state (cheese) is 1; reward is 0 elsewhere. \textbf{b-c.} Examples of prioritized experience replay by $|\text{TD}|$ and value of experience (EVB). The main difference is that EVB only prioritizes the experiences that are associated with the optimal policy; while $|\text{TD}|$ is sensitive to changes in value function and will prioritize non-optimal experiences, such as those associated with north or south. Here squares represent states, triangles represent actions, and experiences associated with the highest priority are highlighted. \textbf{d.} Expected number of replays needed to learn the optimal policy, as the number of grids changes: uniform replay (blue), prioritized by $|\text{TD}|$ (orange), and EVB (green).}
  \label{fig:fig1}
\end{figure*}

\subsection{Value Metrics of Experience in Q-Learning}
For Q-learning, we use Q-function to estimate the true action-value function. A backup over an experience $e_k$ consists of policy evaluation with Bellman operator and greedy policy improvement. As the policy improvement is greedy, we can rewrite value metrics of experience to simpler forms. EVB can be written as follows from (\ref{eq:evb}),
\begin{equation} \label{eq:evb3} 
\text{EVB}(e_k) = \max_a{Q_{\text{new}}(s_k,a)} - \max_a{Q_{\text{old}}(s_k,a)}.
\end{equation}
Note that EVB here is different from that in \cite{Mattar2018}: in our case, EVB is derived from Q-learning; while in their case, EVB is derived from Dyna, a model-based RL algorithm \citep{sutton1990integrated}. Similarly, from (\ref{eq:evb2}), PIV can be written as
\begin{equation} \label{eq:piv} 
\text{PIV}(e_k) = \max_a{Q_{\text{new}}(s_k,a)} - Q_{\text{new}}(s_k, a_{\text{old}}),
\end{equation}
where $a_{\text{old}} = \arg\max_a{Q_{\text{old}}(s_k,a)}$, and EIV can be written as
\begin{equation} \label{eq:eiv} 
\text{EIV}(e_k) = Q_{\text{new}}(s_k,a_{\text{old}}) - Q_{\text{old}}(s_k, a_{\text{old}}).
\end{equation}

\subsection{A Motivating Example}

We illustrate the potential gain of value of experience in a ``Linear Grid-World" environment (Figure~\ref{fig:fig1}a). This environment contains $N$ linearly-aligned grids and 4 actions (north, south, east, west). The rewards are rare: 1 for entering the goal state and 0 elsewhere. The solution for this environment is always choosing east.

We use this example to highlight the difference between prioritization strategies. Three agents perform Q-learning updates on the experiences drawn from the same replay buffer, which contains all the ($4N$) experiences and associated rewards. The first agent replays the experiences uniformly at random, while the other two agents invoke the oracle to prioritize the experiences, which greedily select the experience with the highest $|\text{TD}|$ or EVB respectively. In order to learn the optimal policy, agents need to replay the experiences associated with action east in a reverse order. 

For the agent with random replay, the expected number of replays required is $4N^2$ (Figure~\ref{fig:fig1}d). For the other two agents, prioritization significantly reduces the number of replays required: prioritization with $|\text{TD}|$ requires $4N$ replays, and prioritization with EVB only uses $N$ replays, which is optimal (Figure~\ref{fig:fig1}d). The main difference is that EVB only prioritizes the experiences that are associated with the optimal policy (Figure~\ref{fig:fig1}c), while $|\text{TD}|$ is sensitive to changes in the value function and will prioritize non-optimal experiences: for example, the agent may choose the experiences associated with south or north in the second update, which are not optimal but have the same $|\text{TD}|$ as the experience associated with east (Figure~\ref{fig:fig1}b). Thus, EVB that directly quantifies the value of experience can serve as an optimal priority.

\section{Upper Bounds of Value Metrics of Experience in Q-Learning}
\label{sec:3}

PER \citep{Schaul2016} greatly improves the learning efficiency of DQN. However, the underlying rationale is not well understood. Here, we prove that $|\text{TD}|$ is the upper bound of the value metrics in Q-learning. 

\begin{figure*}[t!]
  \centering
  \includegraphics[width=.85\textwidth, clip=true, trim = 50mm 0mm 50mm 0mm]{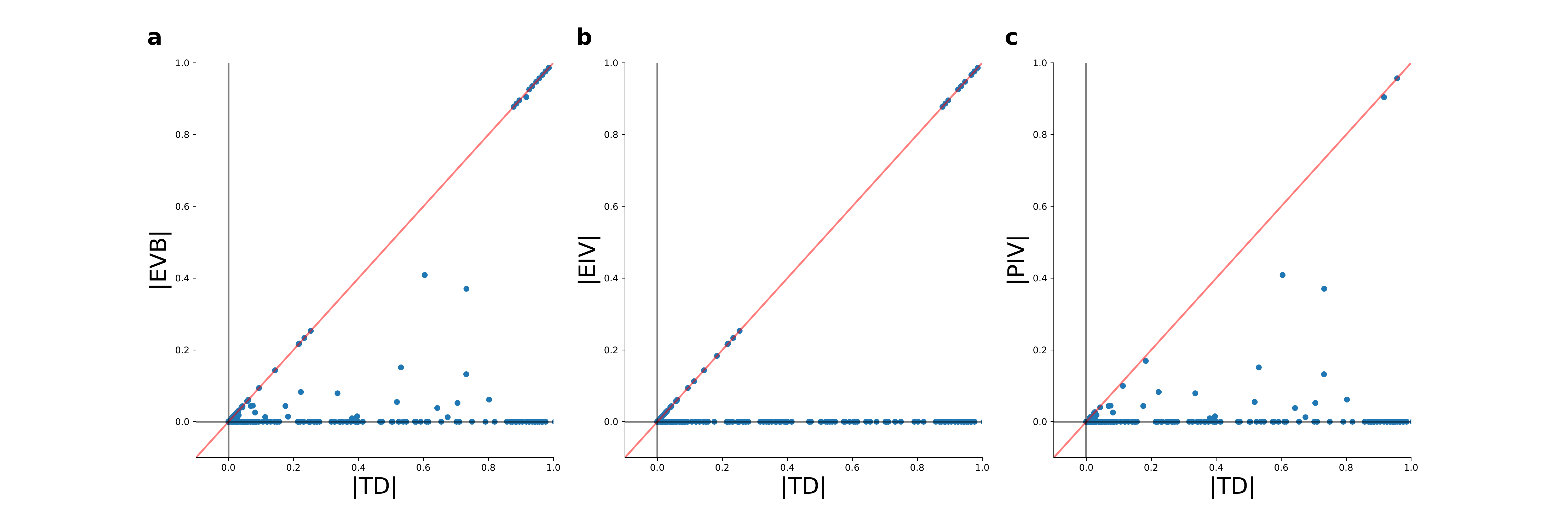}
\caption{The value metrics are upper-bounded by TD errors in Q-learning. \textbf{a-c.} $|\text{TD}|$ \textit{v.s.} $|\text{EVB}|$ (left), $|\text{EIV}|$ (middle) and $|\text{PIV}|$ (right) of a tabular Q-learning agent in a grid-world maze. The red line indicates the identity line. }
  \label{fig:fig2_icml}
\end{figure*}

\begin{theorem}
\label{the:1}
The three value metrics of experience $e_k$ in Q-learning ($|\text{EVB}|$, $|\text{PIV}|$ and $|\text{EIV}|$)  are upper-bounded by $\alpha|\text{TD}(s_k,a_k,r_k,s_k')|$, where $\alpha$ is a step-size parameter.
\end{theorem}
\begin{proof}
From (\ref{eq:evb3}), $|\text{EVB}|$ can be written as
\begin{align}\nonumber
|\text{EVB}(e_k)| 
 &= |\max_a{Q_{\text{new}}(s_k,a)} - \max_a{Q_{\text{old}}(s_k,a)} | \\ \nonumber
 &\leq \max_a{|{Q_{\text{new}}(s_k,a)} - {Q_{\text{old}}(s_k,a)}|} \\ \label{eq:evb-q}
 &\leq \alpha|\text{TD}(s_k,a_k, r_k, s_k')|,
\end{align}
where the second line is from the contraction of max operator. 

Proofs for the upper bounds of $|\text{PIV}|$ and $|\text{EIV}|$ are similar and given in Appendix \ref{sec:app1}.
\end{proof}

In Theorem \ref{the:1}, we prove that $|\text{EVB}|$, $|\text{PIV}|$, and $|\text{EIV}|$ are upper-bounded by $|\text{TD}|$ (scaled by the learning step-size) in Q-learning. To verify the bounds experimentally, we simulated a tabular Q-learning agent in a 5 $\times$ 5 grid-world maze
\footnote{All the codes for the experiments are available at: \url{https://github.com/AmazingAng/VER}.}. 
The agent needs to reach the goal zone by moving one square in any of the four directions (north, south, east, west) each time (further details are described in Appendix \ref{sec:app4}). For each transition, we record the associated TD error and value metrics. As we can see from Figure~\ref{fig:fig2_icml}, all three value metrics of experience are bounded by $|\text{TD}|$. As our theory predicts (see Appendix \ref{sec:app1} for detail), $|\text{EIV}|$ is either equal to $|\text{TD}|$ (if the action of the experience is the optimal action before update) or 0. There is a large proportion of EVBs lies on the identity line, indicating the bound is tight. Moreover, we note that a significant proportion of value metrics lies on the x-axis. Because the value metrics are affected by the ``on-policyness" of the experienced actions, and Q-learning learns a deterministic policy that makes most actions of experiences off-policy. As $|\text{TD}|$ intrinsically tracks the evaluation and policy improvements, it can serve as an appropriate importance metric for past experiences.


\section{Extension to Maximum-Entropy RL}

In this section, we extend our framework to study the relationship between $|\text{TD}|$ and value of experience in maximum-entropy RL, particularly, soft Q-learning. 

\subsection{Soft Q-Learning}
\label{softqlearning}

\begin{figure*}[t!]
  \centering
  \includegraphics[width=.83\textwidth, clip=true, trim = 50mm 20mm 50mm 20mm]{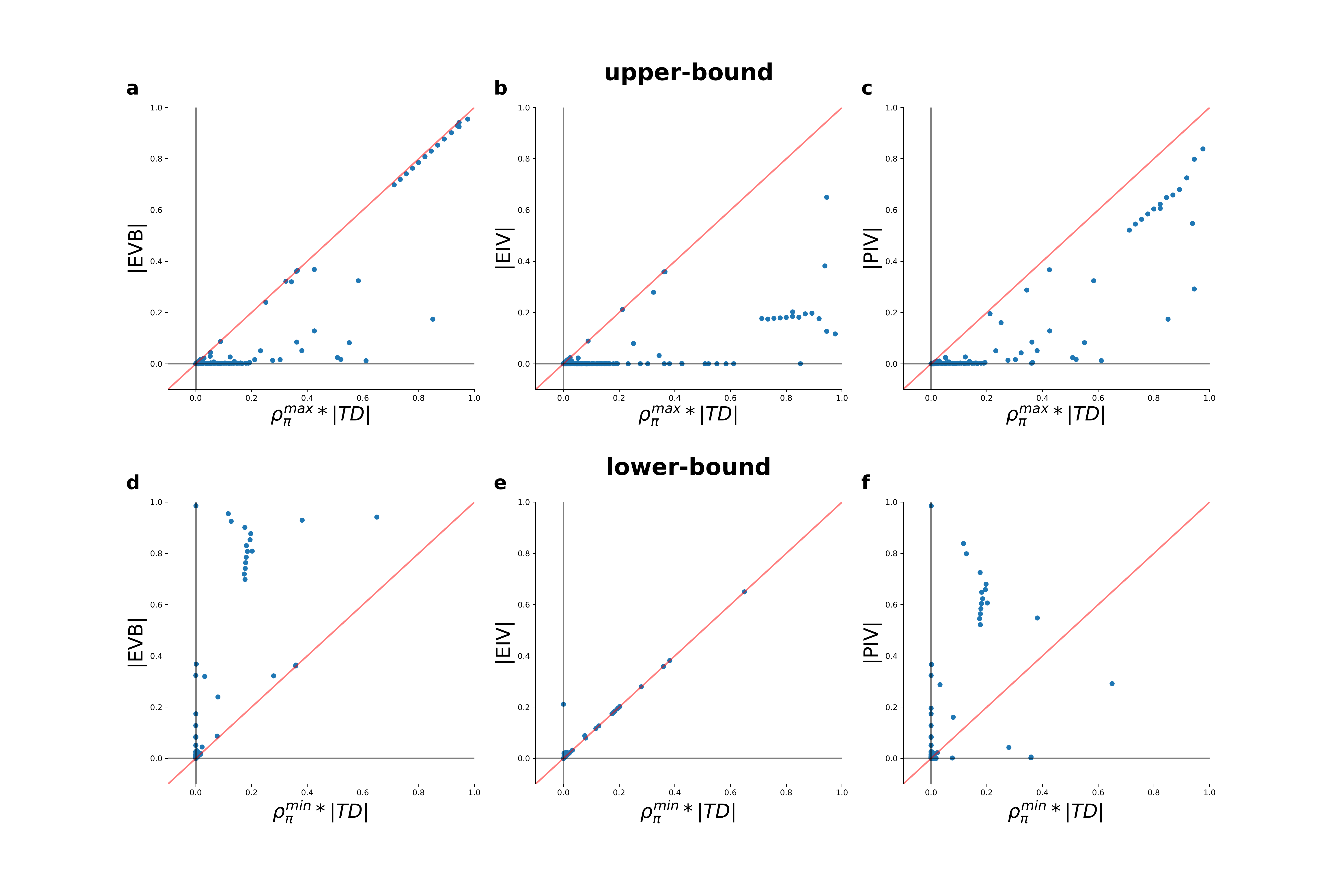}
\caption{The value metrics and their bounds in soft Q-learning. $|\text{EVB}|$ (left), $|\text{EIV}|$ (middle) and $|\text{PIV}|$ (right) as well as their theoretical lower-bound \textbf{a-c.} and upper-bounds \textbf{d-f.} of a tabular soft Q-learning agent in a grid-world maze. The red line indicates the identity line. }
  \label{fig:fig3_icml}
\end{figure*}

Unlike regular RL algorithms, maximum-entropy RL augments the reward with an entropy term: $r + \beta \mathcal{H}(\pi(\cdot|s))$, where $\mathcal{H}(\cdot)$ is the entropy, and $\beta$ is an optional temperature parameter that determines the relative importance of entropy and reward. The goal is to maximize the expected cumulative entropy-augmented rewards. Maximum-entropy RL algorithms have advantages at capturing multiple modes of near optimal policies, better exploration, and better transfer between tasks. 

Soft Q-learning is an off-policy value-based algorithm built on maximum-entropy RL principles \citep{Haarnoja2017, schulman2017equivalence}. Different from Q-learning, the target policy of soft Q-learning is stochastic. During policy iteration, Q-function is updated through soft Bellman operator $\Gamma^\text{soft}$, and the policy is updated to a maximum-entropy policy:
\begin{align*}
Q^\text{soft}_\text{new} (s, a) &= [\Gamma^\text{soft} Q^\text{soft}_\text{old}] (s, a)  = r + \gamma V^\text{soft}_\text{old} (s') \\
\pi_\text{new}(a| s)  &= \text{softmax}_a{(\frac{1}{\beta}Q^\text{soft}_\text{new} (s, a))},
\end{align*}
where $\text{softmax}_i(x) = \exp(x_i)/ \sum_i{\exp(x_i)}$ is the softmax function, and the soft value function $V^\text{soft}_\pi(s)$ is defined as,
\begin{align*}\nonumber
V^\text{soft}_\pi(s) &=  \mathbb{E}_{a}{\{Q^\text{soft}_{\pi}(s,a) - \log(\pi(a|s))\}}\\
&=\beta\log\sum_a\exp(\frac{1}{\beta}Q^\text{soft}_\pi (s, a)).   
\end{align*}
Similar as in Q-learning, the TD error in soft Q-learning (soft TD error) is given by:
\begin{equation*}
\text{TD}^{\text{soft}}(s,a,r,s') =   r + \gamma V^\text{soft}_\text{old} (s') - Q^\text{soft}_\text{old} (s, a).
\end{equation*}

\subsection{Value Metrics of Experience in Maximum-Entropy RL}
Here, we extend the value metrics of experience to soft Q-learning. Similar as (\ref{eq:evb}), EVB for maximum-entropy RL is defined as,
\begin{equation}
\label{eq:evbsoft}
\resizebox{0.9\linewidth}{!}{$
	\begin{split}
		& \text{EVB}^{\text{soft}}(e_k) \\
		& = v^\text{soft}_{\text{new}}(s_k) - v^\text{soft}_{\text{old}}(s_k)  \\ 
		& =  \sum_{a}{\pi_{\text{new}}(a|s_k)\{q^\text{soft}_{\text{new}}(s_k,a)     -\beta\log(\pi_{\text{new}}(a|s_k))\}}  \\
		& \quad -\sum_{a}{\pi_{\text{old}}(a|s_k)\{q^\text{soft}_{\text{old}}(s_k,a)- \beta\log(\pi_{\text{old}}(a|s_k))\}} 
	\end{split}
$}
\end{equation}
$\text{EVB}^{\text{soft}}$ can be separated into $\text{PIV}^{\text{soft}}$ and $\text{EIV}^{\text{soft}}$, which respectively quantify the value of policy and evaluation improvement in soft Q-learning,
\begin{equation}
\label{eq:pivsoft} 	
\resizebox{0.9\linewidth}{!}{$
	\begin{split}
	\text{PIV}^{\text{soft}}(e_k) &= \sum_{a}{\{\pi_{\text{new}}(a|s_k) - \pi_{\text{old}}(a|s_k)\}q^\text{soft}_{\text{new}}(s_k,a)} \\ &\qquad+\beta(H(\pi_{\text{new}}(\cdot|s))-H(\pi_{\text{old}}(\cdot|s_k))),
	\end{split}
$}
\end{equation}
\begin{equation}
	\label{eq:eivsoft}
\resizebox{0.9\linewidth}{!}{$
	\text{EIV}^{\text{soft}}(e_k) = \sum_{a}{\pi_{\text{old}}(a|s_k)[q^\text{soft}_{\text{new}}(s_k,a)-q^\text{soft}_{\text{old}}(s_k,a)]}.
$}
\end{equation}
Value metrics of experience in maximum-entropy RL have similar forms as in regular RL except for the entropy term, because changes in policy lead to changes in the policy entropy and affect the entropy-augmented rewards.

\subsection{Lower and Upper Bounds of Value Metrics of Experience in Soft Q-learning}

We theoretically derive the lower and upper bounds of the value metrics of experience in soft Q-learning.

\begin{theorem}
\label{the:2}
The three value metrics of experience $e_k$ in soft Q-learning ($|\text{EVB}^\text{soft}|$, $|\text{PIV}^\text{soft}|$ and $|\text{EIV}^\text{soft}|$) are upper-bounded by $\rho^\text{max}_\pi * \left| \text{TD}^{\text{soft}}\right|$, where $\rho^\text{max}_\pi = \max\{\pi_{\text{old}}(a_k|s_k), \pi_{\text{new}}(a_k|s_k)\} $ is a policy related term.
\end{theorem}

\begin{proof}
See Appendix \ref{sec:app2}.
\end{proof}

\begin{theorem}
\label{the:3}
For soft Q-learning, $|\text{EVB}^\text{soft}|$ and $|\text{EIV}^\text{soft}|$  (but not $|\text{PIV}^\text{soft}|$) are lower-bounded by $\rho^\text{min}_\pi * \left| \text{TD}^{\text{soft}}\right|$, where $\rho^\text{min}_\pi = \min\{\pi_{\text{old}}(a_k|s_k), \pi_{\text{new}}(a_k|s_k)\} $ is a policy related term.
\end{theorem}

\begin{proof}
See Appendix \ref{sec:app3}.
\end{proof}

\begin{figure*}[t!]
  \centering
  \setlength{\abovecaptionskip}{7pt}
  \includegraphics[width=.75\textwidth, clip=true, trim = 50mm 50mm 50mm 50mm]{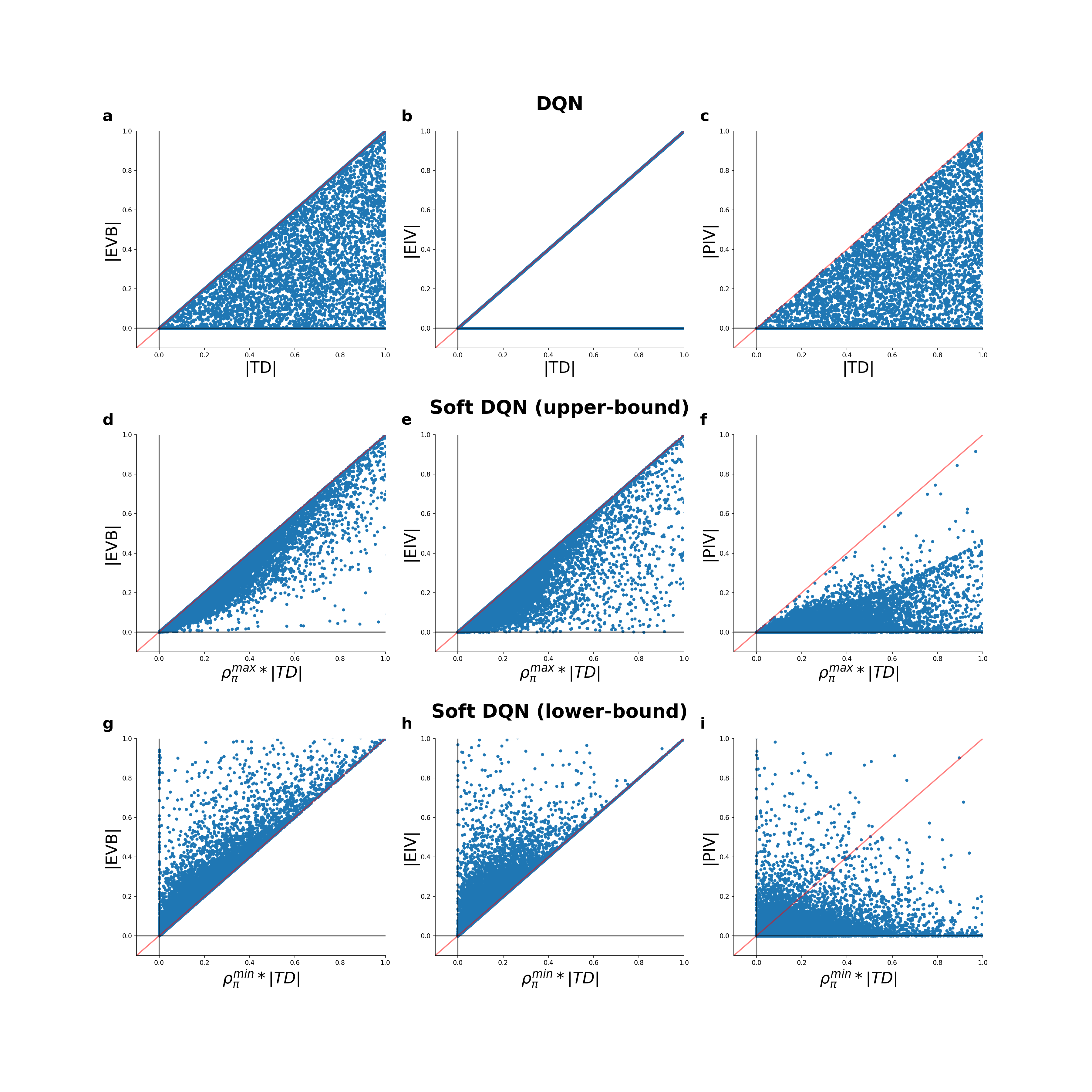}
  \caption{Results of DQN and soft DQN in CartPole. \textbf{a-c.} $|\text{TD}|$ \textit{v.s.} $|\text{EVB}|$ (left), $|\text{EIV}|$ (middle) and $|\text{PIV}|$ (right) in DQN. \textbf{d-f.} Theoretical upper bound and (\textbf{g-i.}) lower bound \textit{v.s.} $|\text{EVB}|$ (left), $|\text{EIV}|$ (middle) and $|\text{PIV}|$ (right) in soft DQN. The red line indicates the identity line.}
  \label{fig:fig4_icml}
\end{figure*}

The lower and upper bounds in soft Q-learning include a policy term with $|\text{TD}|$. The policy related term $\rho_\pi$ quantifies the ``on-policyness'' of the experienced action. And the bounds become tighter as the difference between $\pi_{\text{old}}(a_k|s_k)$ and $\pi_{\text{new}}(a_k|s_k)$ becomes smaller. Surprisingly, the coefficient of the entropy term $\beta$ impacts the bound only through the policy term, which makes it an excellent priority even $\beta$ changes during learning \citep{haarnoja2018soft}.  As $0 \leq \rho^\text{max}_\pi \leq 1$, the value metrics are also upper-bounded by $|\text{TD}|$ alone, which is similar as in Q-learning. However,  as $\pi(a_k|s_k)$ is usually less than $1$,  $|\text{TD}|$ is a looser upper bound in soft Q-learning. 

To verify the bounds in soft Q-learning experimentally, we simulated a tabular soft Q-learning agent in the grid-world maze described previously. From upper panel of Figure~\ref{fig:fig3_icml}, all three value metrics of experience are upper-bounded by $\rho^\text{max}_\pi * |\text{TD}|$. Moreover, from lower panel of Figure~\ref{fig:fig3_icml}, $|\text{EVB}^\text{soft}|$ and $|\text{EIV}^\text{soft}|$  (but not $|\text{PIV}^\text{soft}|$) are lower-bounded by $\rho^\text{min}_\pi * \left| \text{TD}^{\text{soft}}\right|$, supporting our theoretical analysis  (Theorem \ref{the:2} and \ref{the:3}). The proportion of non-zero values of experiences is higher in soft Q-learning than in Q-learning, because, different from greedy policy of Q-learning, soft Q-learning learns a stochastic policy that makes experiences more ``on-policy" and have non-sparse values. In summary, the experimental results support the theoretical bounds of value metrics in tabular soft Q-learning.

\section{Extension to Function Approximation Methods}
\label{sec:FA}

Function approximation methods, which are more powerful and expressive than tabular methods, are effective in solving more challenging tasks, such as the game of Go \citep{silver2016mastering}, video games \citep{Mnih2015} and robotic control \citep{Haarnoja2017}. In these methods, we learn a parameterized Q-function $Q(s,a; \theta_t)$, where the parameters are updated on experience $e_k$ through gradient-based method,
\begin{equation}\nonumber
\theta_{t+1} = \theta_{t} + \alpha \text{TD} \nabla_{\theta_{t}} Q(s_k,a_k; \theta_t)),
\end{equation}
where $\alpha$ is the learning rate, the TD error is defined as:
\begin{equation}\nonumber
\text{TD}= Q_{\text{target}}(s_k,a_k) - Q(s_k,a_k;\theta_t),
\end{equation}
and the target Q-value $Q_{\text{target}}$ is defined as
\begin{equation}\nonumber
Q_{\text{target}}(s_k,a_k) = r_k + \gamma\max_{a'}{Q(s'_k,a'; \theta_t)}.
\end{equation}
As $\alpha$ in function approximation Q-learning is usually very small, for each update, the parameterized function moves to its target by a small amount. 

Our framework can be extended to function approximation methods by slightly modifying the definition of the value metrics of experience. Note that if we apply the original definition of EVB in (\ref{eq:evb3}) directly to function approximation methods, the Q-function after the update $Q_{\text{new}}(s,a)= Q(s,a;\theta_{t+1})$ involves gradient-based update, which complicates the analysis and breaks the inequalities derived in the tabular case. As a remedy, we replace $Q(s,a; \theta_{t+1})$ by the target Q-value $Q_{\text{target}}(s,a)$ in the value metrics of experience (\ref{eq:evb}-\ref{eq:eiv}) and (\ref{eq:evbsoft}-\ref{eq:eivsoft}). The intuition is simple: the value is defined by the cause of the update (target Q-value), but not the result of the update through gradient-based method. Moreover, this modification allows our theory to apply to all function approximation methods, regardless the specific forms of the function approximator (linear function or neural networks). After the modifications, the value metrics of experience have similar form as the tabular case, and all theorems derived in the tabular case can be applied to function approximation methods.

To test whether our theoretical predictions hold in function approximation methods, we simulated one DQN (Deep-Q network) agent and one soft DQN (DQN with soft update) agent in CartPole environment, where the goal is to keep the pole balanced by moving the cart forward and backward (further details are described in Appendix~\ref{sec:app4}). From Figure~\ref{fig:fig4_icml}, all value metrics of experience in DQN (Figure~\ref{fig:fig4_icml}a-c) and soft DQN (Figure~\ref{fig:fig4_icml}d-f) are bounded by the theoretical upper bounds. For DQN, $|\text{EVB}|$ and $|\text{PIV}|$ are uniformly distributed in the bounded area, while $|\text{EIV}|$ are equal to $|\text{TD}|$ or $0$. Results are different in soft DQN, where $|\text{EVB}|$ and $|\text{EIV}|$ are distributed more closely towards the theoretical upper bounds, suggesting the upper bound in soft Q-learning is tighter. Moreover, (Figure~\ref{fig:fig4_icml}g-i) shows the $|\text{EVB}|$ and $|\text{EIV}|$ are lower-bounded by $\rho^\text{min}_\pi * \left| \text{TD}^{\text{soft}}\right|$, while $|\text{PIV}|$ are not. The experimental results confirm the bounds of value metrics hold for function approximation methods.

\section{Experiments on Atari Games}

\begin{figure}[t!]
  \centering
  \includegraphics[width=.85\columnwidth, clip=true, trim = 10mm 20mm 10mm 25mm]{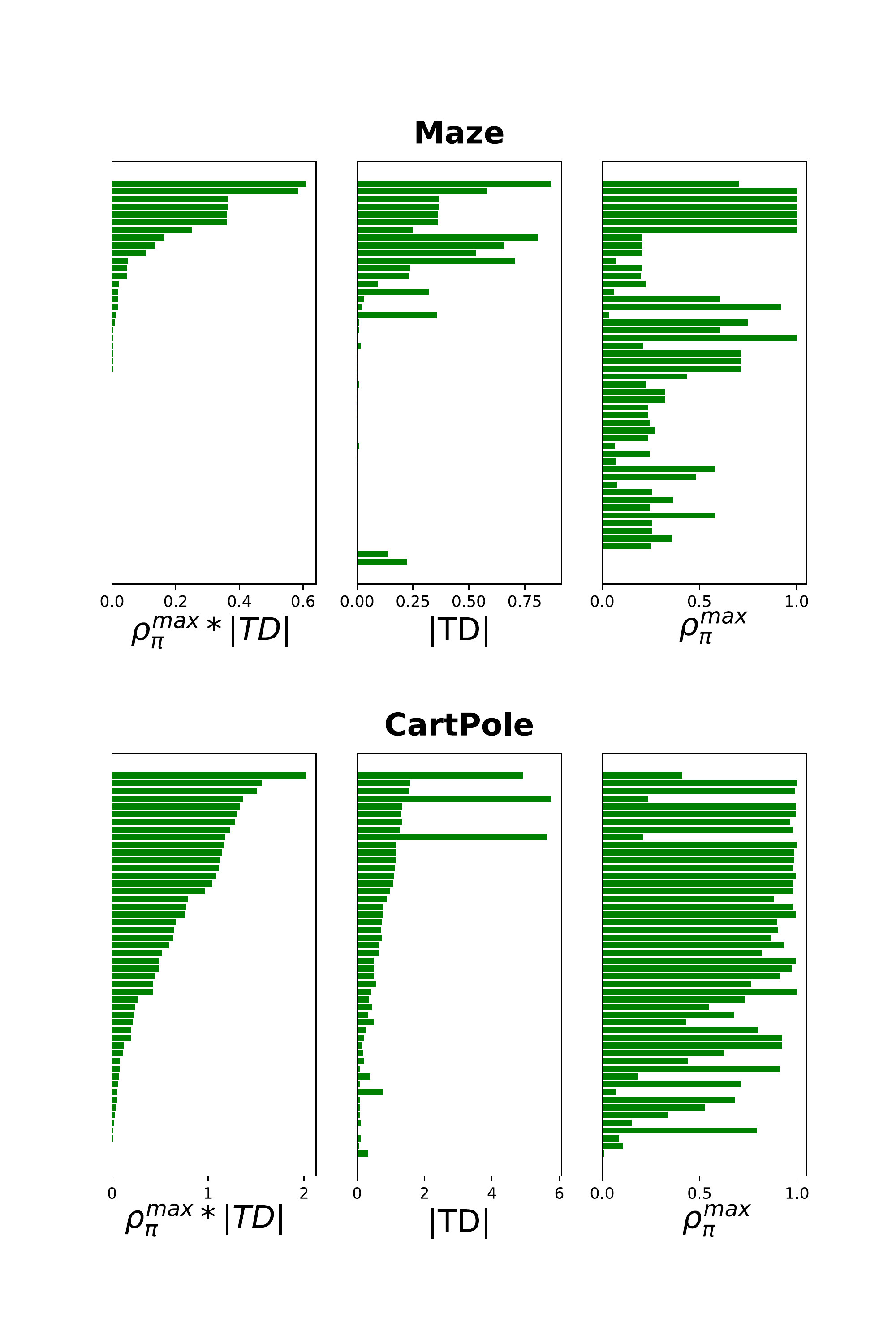}
  \caption{Illustration on the difference between $|\text{TD}|$ and theoretical upper-bound for value metrics in soft Q-learning. Depicted are the theoretical upper bound (left), $|\text{TD}|$ (middle), and the policy term (right) of 50 experiences from the replay buffer in the grid-world maze (upper panel) and CartPole (lower panel), ordered by the theoretical upper bound.}
  \label{fig:supp2}
  \vspace*{-0.2cm}
\end{figure}

\begin{figure*}[t!]
  \centering
  \includegraphics[width=0.75\textwidth]{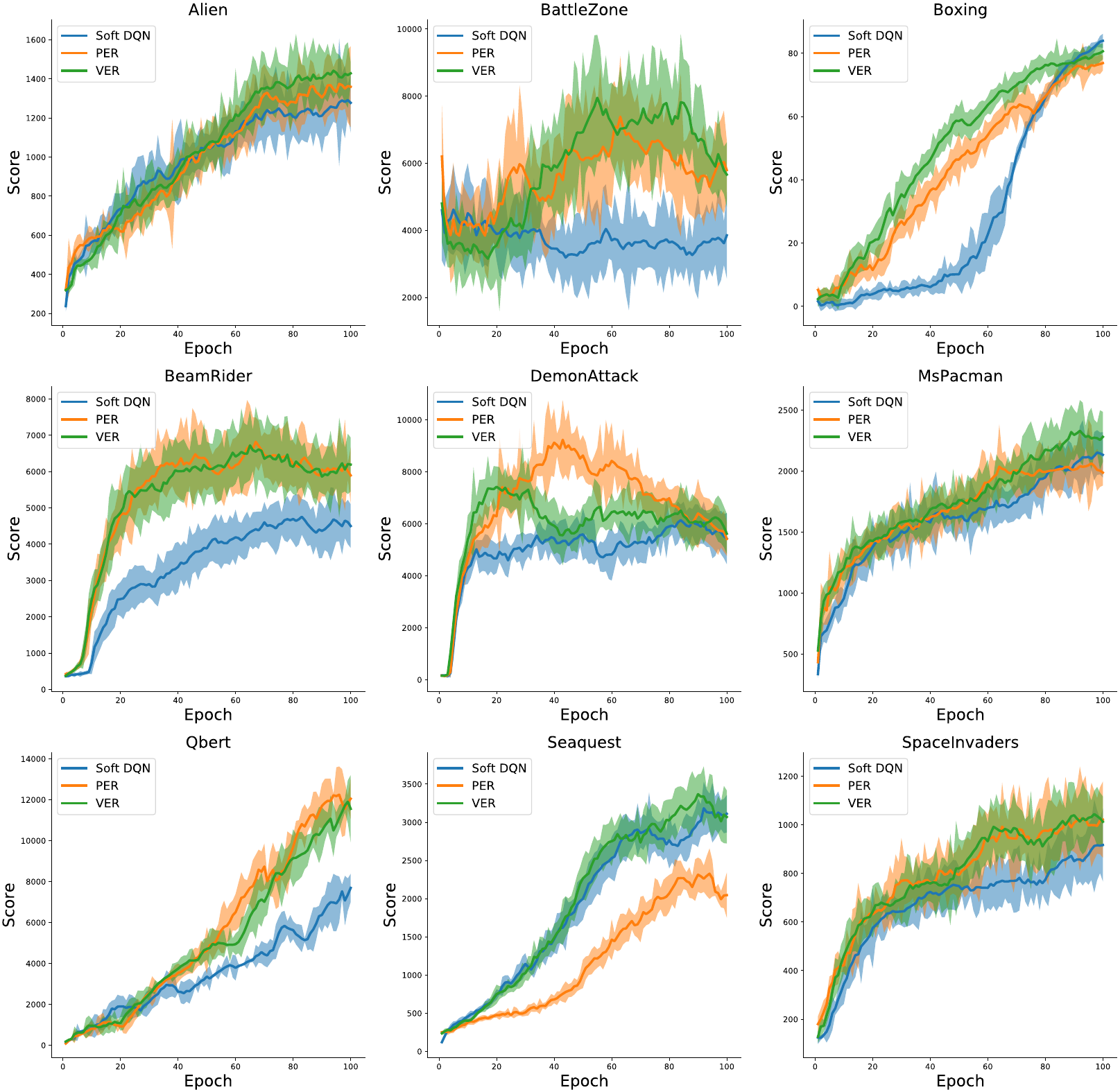}
  \caption{Learning curve of soft DQN (blue lines), and soft DQN with prioritized experience replay in term of soft TD error (PER, orange lines) and the theoretical upper bound of value metrics of experience (VER, green lines) on Atari games. Solid lines are average return over 8 evaluation runs and shaded area is the standard error of the mean.}
  \label{fig:fig4}
\end{figure*}

The theoretical upper bound ($\rho^\text{max}_\pi * |\text{TD}|$) of value metrics balances the prediction error and ``on-policyness" of the experience. To better illustrate the difference between the theoretical upper bound and $|\text{TD}|$, we randomly drew 50 experiences from the grid-world maze and CartPole experiments. From Figure~\ref{fig:supp2}, we can see that the experiences with highest theoretical upper bounds are associated with higher $|\text{TD}|$ and "on-policyness". To investigate whether the theoretical upper bound can serve as an appropriate priority for experience replay in soft Q-learning, we compare the performance of soft DQN with different prioritization strategies: uniform replay, prioritization with $|\text{TD}|$ or the theoretical upper bound ($\rho^\text{max}_\pi * \left| \text{TD}^{\text{soft}}\right|$), which are denoted by soft DQN, PER and VER (valuable experience replay) respectively. This set of experiments consists of 9 selected Atari 2600 games according to \cite{schulman2017equivalence}, which balances generality of the games and limited compute power. We closely follow the experimental setting and network architecture outlined by \cite{Mnih2015}. For each game, the network is trained on a single GPU for 40M frames, or approximately 5 days. More details for the settings and hyperparameters are available in Appendix \ref{sec:app4}.


Figure~\ref{fig:fig4} shows that soft DQN prioritized by $|\text{TD}|$ or the theoretical upper bound substantially outperforms uniform replay in most of the games. On average, soft DQN with PER or VER outperform vanilla soft DQN by 11.8\% or 18.0\% respectively. Moreover, VER shows higher convergence speed and outperforms PER in most of the games (8.47\% on average), which suggest that a tighter upper bound on value metrics improves the performance of experience replay. These results suggest that the theoretical upper bound can serve as an appropriate priority for experience replay in soft Q-learning.

\section{Discussion}

In this work, we formulate a framework to study relationship between the importance of experience and $|\text{TD}|$. To quantify the importance of experience, we derive three value metrics of experience: expected value of backup, policy evaluation value, and policy improvement value. For Q-learning, we theoretically show these value metrics are upper-bounded by $|\text{TD}|$. Thus, $|\text{TD}|$ implicitly tracks the value of the experience, which leads to high sample efficiency of PER. Furthermore, we extend our framework to maximum-entropy RL, by showing that these value metrics are lower and upper-bounded by the product of a policy term and $|\text{TD}|$. Experiments in grid-world maze and CartPole support our theoretical results for both tabular and function approximation RL methods, showing that the derived bounds hold in practice. Moreover, we show that experience replay using the upper bound as a priority improves maximum-entropy RL (\textit{i.e.}, soft DQN) in Atari games.

By linking $|\text{TD}|$ and value of experience, two important quantities in learning, our study has the following implications. First, from a machine learning perspective, our study provides a framework to derive appropriate priorities of experience for different algorithms, with possible extension to batch RL \citep{fu2020d4rl} and N-step learning \citep{hessel2017rainbow}. Second, for neuroscience, our work provides insight on how brain might encode the importance of experience. Since $|\text{TD}|$ biologically corresponds to the reward prediction-error signal in the dopaminergic system \citep{schultz1997neural, glimcher2011understanding} and implicitly tracks the value of the experience, the brain may account on it to differentiate important experiences.

\bibliography{main_arxiv.bib}
\bibliographystyle{icml_arxiv}

\appendix
\onecolumn

\section{Appendix}
\subsection{Proof of Theorem \ref{the:1}}
\label{sec:app1}
In this section, we derive the upper bounds of $|\text{PIV}|$ and $|\text{EIV}|$ in Q-learning. $|\text{PIV}|$ can be written as
\begin{align}\small \nonumber
|\text{PIV}(e_k)|
&= |\max_a{Q_{\text{new}}(s_k,a)} - Q_{\text{new}}(s_k, \arg\max_a{Q_{\text{old}}(s_k,a)})| \\ \nonumber
&= \max_a{Q_{\text{new}}(s_k,a)} - Q_{\text{new}}(s_k, \arg\max_a{Q_{\text{old}}(s_k,a)}) \\\label{eq:piv-mid}
&= \max_a{Q_{\text{new}}(s_k,a)} -  \max_a{Q_{\text{old}}(s_k,a)} - \1_\mathrm{a_{\text{old}}= a_k} \alpha \text{TD}(s_k, a_k, r_k, s_k') 
\end{align}


where the second line is from that the change in Q-function following greedy policy improvement is greater or equal to 0, and the third line is from the update of Q-function. For $\text{TD}(s_k, a_k, r_k, s_k') \geq 0$, we have
\begin{equation}\nonumber\small
0 \leq \max_a{Q_{\text{new}}(s_k,a)} -  \max_a{Q_{\text{old}}(s_k,a)} \leq \alpha\text{TD}(s_k, a_k, r_k, s_k').
\end{equation}
And for $\text{TD}(s_k, a_k, r_k, s_k') \leq 0$, we have
\begin{equation}\nonumber\small
\max_a{Q_{\text{new}}(s_k,a)} -  \max_a{Q_{\text{old}}(s_k,a)} \leq 0.
\end{equation}
Bring above inequalities to \ref{eq:piv-mid}, we have
\begin{equation}\small\label{eq:piv-q}
|\text{PIV}(e_k)| \leq \alpha|\text{TD}(s_k,a_k,r_k,s_k')|
\end{equation}

Similarly, $|\text{EIV}|$ can be written as follows,
\begin{align}\small\nonumber
|\text{EIV}(e_k)| 
&= |Q_{\text{new}}(s_k,a_{\text{old}}) - Q_{\text{old}}(s_k, a_{\text{old}})| \\ \nonumber
&= \1_\mathrm{s =s_k, a_{\text{old}}= a_k} \alpha|\text{TD}(s_k, a_k, r_k, s_k')|\\ \label{eq:eiv-q}
 &\leq \alpha|\text{TD}(s_k,a_k, r_k, s_k')|
\end{align}
For equations (\ref{eq:evb-q}) and (\ref{eq:piv-q}), the equality is reached if the experienced action is the same as the best action before and after the update. For (\ref{eq:eiv-q}), the equality is met if the experienced action is the best action before update. 
\hfill $\square$

\subsection{Proof of Theorem \ref{the:2}}
\label{sec:app2}
In this section, we derive upper bounds of value metrics of experience in soft Q-learning. For soft-Q learning, $|\text{EVB}^\text{soft}|$ can be written as
\begin{align*}\small
|\text{EVB}^\text{soft}(e_k)| &=  |\beta\log\sum_a\exp(\frac{1}{\beta}Q^\text{soft}_{\text{new}} (s_k, a)) -  \beta\log\sum_a\exp(\frac{1}{\beta}Q^\text{soft}_{\text{old}} (s_k, a)) | \\
 &= 
 |\beta\log\sum_a\exp{(\frac{1}{\beta}(Q^\text{soft}_{\text{old}} (s_k, a) +  \1_\mathrm{ a= a_k}  \text{TD}^{\text{soft}}))} -
 \beta\log\sum_a\exp{\frac{1}{\beta}Q^\text{soft}_{\text{old}} (s_k, a)} | .
\end{align*}

Let us define the LogSumExp function $F(\vec{x}) = \beta\log\sum_i\exp{(\frac{x_i}{\beta})}$. The LogSumExp function $F(\vec{x})$ is convex, and is strictly and monotonically increasing everywhere in its domain \citep{elghaouilivebookopt2018}. The partial derivative of $F(\vec{x})$ is a softmax function
\begin{equation}\nonumber\small
\frac{\partial F(\vec{x})}{\partial x_i} = \text{softmax}_i{(\frac{1}{\beta}\vec{x})} \geq 0,
\end{equation}
which takes the same form as the policy of soft Q-learning. For $\epsilon < 0$, we have:
\begin{equation}\nonumber\small
\epsilon  \frac{\partial F(x_1,...,x_i,...)}{\partial x_i}\leq F(x_1,...,x_i +\epsilon, ...)-F(x_1,...,x_i , ...)\leq 0.
\end{equation}

Similarly, for $\epsilon \geq 0$, we have,
\begin{equation}\nonumber\small
0 \leq F(x_1,...,x_i +\epsilon, ...)-F(x_1,...,x_i , ...)\leq \epsilon  \frac{\partial F(x_1,...,x_i+ \epsilon,...)}{\partial x_i}.
\end{equation}
By substituting $x_i$ by $Q^\text{soft}_{\text{old}} (s_k, a_k)$ and $\epsilon$ by $\text{TD}^{\text{soft}}$, and rewriting partial derivative of $F(\vec{x})$ into policy form, we have following inequalities. For $\text{TD}^{\text{soft}} \leq 0$
\begin{equation}\nonumber\small
\pi_{\text{old}}(a_k|s_k) \text{TD}^{\text{soft}} 
\leq \beta\log\sum_a\exp(\frac{1}{\beta}Q^\text{soft}_{\text{new}} (s_k, a)) -  \beta\log\sum_a\exp(\frac{1}{\beta}Q^\text{soft}_{\text{old}} (s_k, a)) \leq 0 
\end{equation}
Similarly, for $\text{TD}^{\text{soft}} > 0$, we have :
\begin{equation}\nonumber\small
0 \leq \beta\log\sum_a\exp(\frac{1}{\beta}Q^\text{soft}_{\text{new}} (s_k, a)) -  \beta\log\sum_a\exp(\frac{1}{\beta}Q^\text{soft}_{\text{old}} (s_k, a)) 
\leq \pi_{\text{new}}(a_k|s_k) \text{TD}^{\text{soft}} 
\end{equation}

Thus, we have the upper bounds of $|\text{EVB}^\text{soft}|$,
\begin{equation}\nonumber\small
\left|\text{EVB}^\text{soft}(e_k)\right| \leq \max\{\pi_{\text{old}}(a_k|s_k), \pi_{\text{new}}(a_k|s_k)\} *\left|\text{TD}^{\text{soft}} \right|.
\end{equation}

For $|\text{PIV}^\text{soft}|$, we have,
\begin{align*}\small
|\text{PIV}^\text{soft}(e_k)| &=  |\sum_{a}{\pi_{\text{new}}(a|s)\{Q^\text{soft}_{\text{new}}(s_k,a) - \beta\log(\pi_{\text{new}}(a|s))\}} \\
& \quad\quad\quad\quad -\sum_{a}{\pi_{\text{old}}(a|s)\{Q^\text{soft}_{\text{new}}(s_k,a)- \beta\log(\pi_{\text{old}}(a|s))\}}| \\
&= \sum_{a}{\pi_{\text{new}}(a|s)\{Q^\text{soft}_{\text{new}}(s_k,a) - \beta\log(\pi_{\text{new}}(a|s))\}} \\
& \quad\quad\quad\quad -\sum_{a}{\pi_{\text{old}}(a|s)\{Q^\text{soft}_{\text{old}}(s_k,a)- \beta\log(\pi_{\text{old}}(a|s))\}} - \pi_{\text{old}}(a_k|s) \text{TD}^{\text{soft}}  \\
&= \beta\log\sum_a\exp{\frac{Q^\text{soft}_{\text{new}} (s_k, a)}{\beta}} -\beta\log\sum_a\exp{\frac{Q^\text{soft}_{\text{old}} (s_k, a)}{\beta}}-\pi_{\text{old}}(a_k|s) \text{TD}^{\text{soft}},
\end{align*}
where the second line is because the policy improvement value is always greater than or equal to 0, and the third line is by reordering the equation. 

For $\text{TD}^{\text{soft}} > 0$, we have:
\begin{equation}\nonumber\small
0 \leq \beta\log\sum_a\exp{\frac{Q^\text{soft}_{\text{new}} (s_k, a)}{\beta}} -\beta\log\sum_a\exp{\frac{Q^\text{soft}_{\text{old}} (s_k, a)}{\beta}}-\pi_{\text{old}}(a_k|s) \text{TD}^{\text{soft}}
\leq \pi_{\text{new}}(a_k|s_k) \text{TD}^{\text{soft}} 
\end{equation}
For $\text{TD}^{\text{soft}} \leq 0$, we have:
\begin{equation}\nonumber\small
0 \leq \beta\log\sum_a\exp{\frac{Q^\text{soft}_{\text{new}} (s_k, a)}{\beta}} -\beta\log\sum_a\exp{\frac{Q^\text{soft}_{\text{old}} (s_k, a)}{\beta}}-\pi_{\text{old}}(a_k|s) \text{TD}^{\text{soft}}
\leq \pi_{\text{old}}(a_k|s_k) \text{TD}^{\text{soft}} 
\end{equation}\nonumber
Thus, we have the upper bounds of $ \left|\text{PIV}^\text{soft} \right|$:
\begin{equation}\nonumber\small
\left|\text{PIV}^\text{soft}(e_k) \right|\leq  \max\{\pi_{\text{old}}(a_k|s_k), \pi_{\text{new}}(a_k|s_k)\} *\left|\text{TD}^{\text{soft}} \right|
\end{equation}

Also, for $|\text{EIV}^\text{soft}|$, we have:
\begin{align*}\small
|\text{EIV}^\text{soft}(e_k)| &=  | \sum_{a}{\pi_{\text{old}}(a|s)[Q^\text{soft}_{\text{new}}(s_k,a)-Q^\text{soft}_{\text{old}}(s_k,a)]} | \\
&= \pi_{\text{old}}(a|s) *\left|\text{TD}^{\text{soft}} \right| \\
&\leq \max\{\pi_{\text{old}}(a_k|s_k), \pi_{\text{new}}(a_k|s_k)\} *\left|\text{TD}^{\text{soft}} \right|
\end{align*}

There is no lower bound of the similar form for $|\text{PIV}^\text{soft}|$.
\hfill $\square$

\subsection{Proof of Theorem \ref{the:3}}
\label{sec:app3}
In this section, we derive lower bounds of value metrics of experience in soft Q-learning. Similar as deriving upper bounds in Appendix \ref{sec:app2}, we derive the lower bounds for $|\text{EVB}|$ using the the LogSumExp function $F(\vec{x}) = \beta\log\sum_i\exp{(\frac{x_i}{\beta})}$. For $\epsilon < 0$, we have:
\begin{equation}\nonumber\small
F(x_1,...,x_i +\epsilon, ...)-F(x_1,...,x_i , ...) \leq \epsilon  \frac{\partial F(x_1,...,x_i +\epsilon,...)}{\partial x_i} \leq 0.
\end{equation}
Similarly, for $\epsilon \geq 0$, we have,
\begin{equation}\nonumber\small
F(x_1,...,x_i +\epsilon, ...)-F(x_1,...,x_i , ...)\geq \epsilon  \frac{\partial F(x_1,...,x_i,...)}{\partial x_i} \geq 0.
\end{equation}
By substituting $x_i$ by $Q^\text{soft}_{\text{old}} (s_k, a_k)$ and $\epsilon$ by $\text{TD}^{\text{soft}}$, and rewriting partial derivative of $F(\vec{x})$ into policy form, we have following inequalities. For $\text{TD}^{\text{soft}} \leq 0$
\begin{equation}\nonumber\small
 \beta\log\sum_a\exp(\frac{1}{\beta}Q^\text{soft}_{\text{new}} (s_k, a)) -  \beta\log\sum_a\exp(\frac{1}{\beta}Q^\text{soft}_{\text{old}} (s_k, a)) \leq \pi_{\text{new}}(a_k|s_k) \text{TD}^{\text{soft}} \leq 0 
\end{equation}
Similarly, for $\text{TD}^{\text{soft}} > 0$, we have :
\begin{equation}\nonumber\small
\beta\log\sum_a\exp(\frac{1}{\beta}Q^\text{soft}_{\text{new}} (s_k, a)) -  \beta\log\sum_a\exp(\frac{1}{\beta}Q^\text{soft}_{\text{old}} (s_k, a)) 
\geq \pi_{\text{old}}(a_k|s_k) \text{TD}^{\text{soft}} \geq 0
\end{equation}
Thus, we have the lower bounds of $|\text{EVB}^\text{soft}|$,
\begin{equation}\nonumber\small
\left|\text{EVB}^\text{soft}(e_k)\right| \geq \min\{\pi_{\text{old}}(a_k|s_k), \pi_{\text{new}}(a_k|s_k)\} *\left|\text{TD}^{\text{soft}} \right|.
\end{equation}

For $|\text{EIV}^\text{soft}|$, we have:
\begin{align*}\small
|\text{EIV}^\text{soft}(e_k)| &=  | \sum_{a}{\pi_{\text{old}}(a|s)[Q^\text{soft}_{\pi_{\text{new}}}(s_k,a)-Q^\text{soft}_{\pi_{\text{old}}}(s_k,a)]} | \\
&= \pi_{\text{old}}(a|s) *\left|\text{TD}^{\text{soft}} \right| \\
&\geq \min\{\pi_{\text{old}}(a_k|s_k), \pi_{\text{new}}(a_k|s_k)\} *\left|\text{TD}^{\text{soft}} \right|
\end{align*}
\hfill $\square$

\subsection{Experimental Details}
\label{sec:app4}

\subsubsection{Grid-world Maze}
\begin{figure*}[t!]
  \centering
  \includegraphics[width=0.7\textwidth]{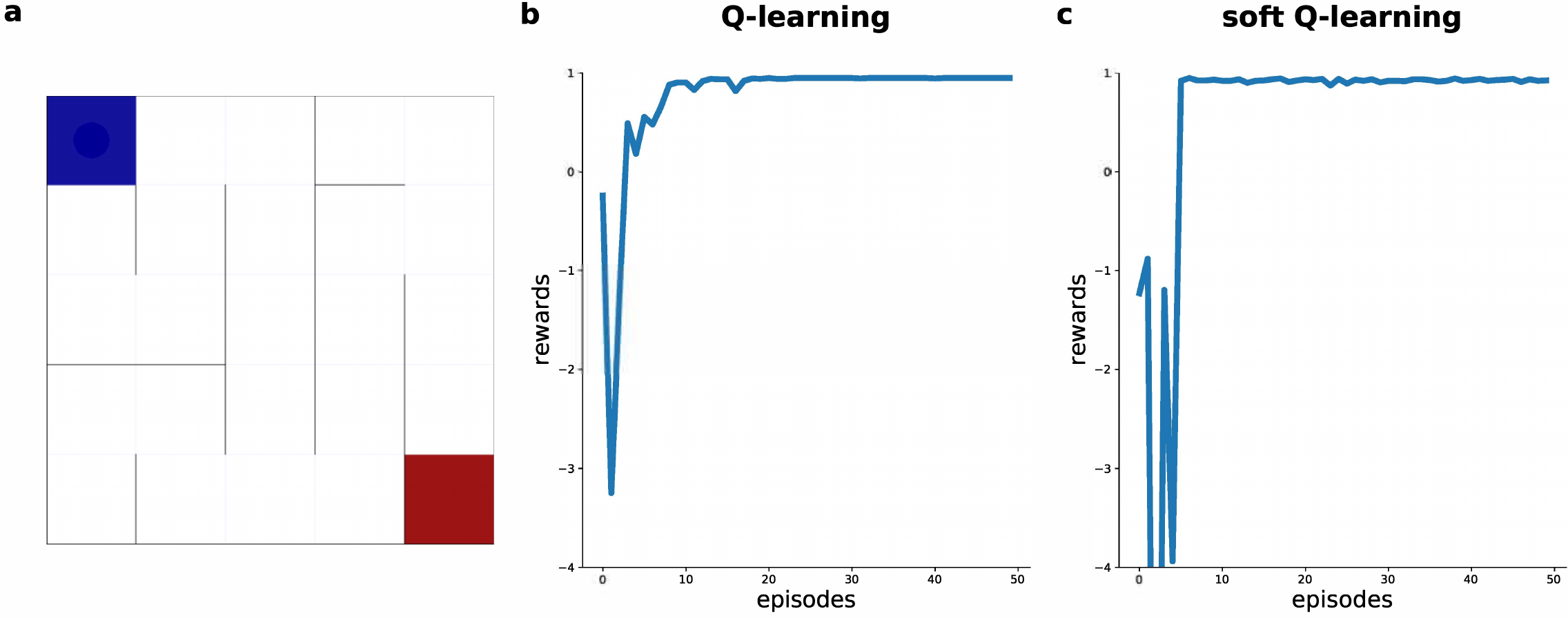}
  \caption{Grid-world maze environment and learning curves.}
  \label{fig:maze}
  \vspace*{-0.2cm}
\end{figure*}
 
For the grid-world maze experiments in section \ref{sec:3}, we use a maze environment of a $5 \times 5$ square with walls, as depicted in Figure~\ref{fig:maze}a. The agent needs to reach the goal zone in the bottom-right corner. At each time step, the agent can choose to move one square in any of the four directions (north, south, east, west). If the move is blocked by a wall or the border of the maze, the agent stays in place. Every time step, the agent gets a reward of $-0.004$ or $1$ if it enters the goal zone and the episode ends. The discount factor is $0.99$ throughout the experiments. 
For these experiments, we use a tabular setting for Q-learning and soft Q-learning according to section \ref{qlearning} and \ref{softqlearning}. For Q-learning, the behavior policy is $\epsilon$-greedy, where $\epsilon$ decays exponentially from $1$ to $0.001$ during training. And we set learning step size $\alpha=1$. For soft Q-learning, the temperature parameter $\beta$ is set to $100$. Total trial number is 50 for each algorithm. During training, both algorithms successfully solve the maze game, see Figure~\ref{fig:maze}b-c for the learning curves.

\subsubsection{CartPole}
For CartPole, the goal is to keep the pole balanced by moving the cart forward and backward for 200 steps. We test our theoretical prediction on DQN and soft-DQN (DQN with soft-update). For DQN, we implement the model according to \cite{Mnih2015}, where we replace the original Q-network with a two-layer MLP, with 256 Relu neurons in each layer. The $\epsilon$ in $\epsilon$-greedy policy decays exponentially from $1$ to $0.01$ for the first $10,000$ steps, and remains $0.01$ afterwards. For soft-DQN, all settings are the same with DQN, except for two modifications: for policy evaluation, the (soft) Q-network is updated according to the soft TD error; the policy follows maximum-entropy policy, calculated as the softmax of the soft Q values (see section \ref{softqlearning}). The temperature parameter $\beta$ is set to $0.5$. For both algorithms, the discount factor is $0.99$, the learning rate is $0.005$, experience buffer size is 1000, the batch size is 16 and total environment interaction is $50,000$.

\subsubsection{Atari Games}
For this set of experiments, we compare the performance of vanilla soft DQN and soft DQN with PER, where we use $|\text{TD}|$ and the theoretical upper bound as priorities \citep{Schaul2016}, respectively denoted as PER and VER (valuable experience replay). We select 9 Atari games for the experiments: Alien, BattleZone, Boxing, BeamRider, DemonAttack, MsPacman, Qbert, Seaquest and SpaceInvaders. The vanilla soft DQN is similar to that described in the above section, but the Q-network has the same architecture as in \cite{Mnih2015}. We implement PER on soft-DQN according to \cite{Schaul2016}. For all algorithms, the temperature parameter $\beta$ is $0.05$, the discount factor is $0.99$, the learning rate is $1\mathrm{e}{-4}$, experience buffer size is $1\text{M}$, the batch size is 32, total environment interaction is $50,000$. For PER or VER, the parameters for importance sampling are $\alpha_{\text{IS}} = 0.4$ and $\beta_{\text{IS}} = 0.6$. For each game, the network is trained on a single GPU for 40M frames, or approximately 5 days.

\end{document}

%% file: math_commands.tex
\usepackage{amsmath,amsfonts,bm}

\def\eqref#1{equation~\ref{#1}}

\def\1{\bm{1}}

\DeclareMathAlphabet{\mathsfit}{\encodingdefault}{\sfdefault}{m}{sl}
\SetMathAlphabet{\mathsfit}{bold}{\encodingdefault}{\sfdefault}{bx}{n}

\DeclareMathOperator*{\argmax}{arg\,max}